\documentclass[11pt]{article}
\usepackage{fullpage}
\usepackage{amsmath,amssymb,amsthm,amsfonts,dsfont,mathtools,latexsym,natbib}
\usepackage{hyperref}
\usepackage{algorithm}
\usepackage[noend]{algpseudocode}
\usepackage{xcolor}
\newcommand{\argmax}{\mathrm{argmax}}
\newcommand{\poly}{\mathrm{poly}}

\def\cA{\mathcal{A}}
\def\cD{\mathcal{D}}
\def\cE{\mathcal{E}}
\def\cS{\mathcal{S}}

\newcommand{\abs}[1]{\left|#1\right|}
\newcommand{\expect}{\mathbb{E}}

\newcommand{\indict}{\mathbb{I}}
\newcommand{\states}{\mathcal{S}}
\newcommand{\trans}{P}
\newcommand{\actions}{\mathcal{A}}
\newcommand{\mdp}{M}
\newcommand{\mdps}{\mathcal{M}}

\newtheorem{thm}{Theorem}[section]
\newtheorem{theorem}{Theorem}[section]

\newtheorem{lemma}[theorem]{Lemma}
\newtheorem{asmp}{Assumption}[section]
\newtheorem{defn}{Definition}[section]
\newtheorem{definition}{Definition}[section]

\theoremstyle{definition}

\title{
Is Long Horizon Reinforcement Learning More Difficult\\
Than Short Horizon Reinforcement Learning?
	}

\date{}
\author{
	 Ruosong Wang\thanks{Equal contribution}\\
	 Carnegie Mellon University \\
	 \texttt{ruosongw@andrew.cmu.edu}\\
\and Simon S. Du$^*$
\\
Institute for Advanced Study\\
\texttt{ssdu@ias.edu}\\
\and Lin F. Yang$^*$ \\
University of California, Los Angles\\
\texttt{linyang@ee.ucla.edu}\\
\and Sham M. Kakade \\	
University of Washington, Seattle and Microsoft Research\\
\texttt{sham@cs.washington.edu}
	}

\begin{document}
\maketitle

\begin{abstract}

Learning to plan for long horizons is a central challenge in episodic
reinforcement learning problems. A fundamental question is to
understand how the difficulty of the problem scales as the
horizon increases. Here the natural measure of sample complexity is a normalized one: we
are interested in the \emph{number
of episodes} it takes to provably discover a policy whose value is
$\varepsilon$ near to that of the optimal value, where the value is measured
by the \emph{normalized} cumulative reward in each episode.  
In a COLT 2018 open problem, 
Jiang and Agarwal conjectured that, for tabular, episodic
reinforcement learning problems, there exists a sample complexity lower
bound which exhibits a polynomial dependence on the horizon --- a
conjecture which is consistent with all known sample complexity upper bounds.
This work refutes this conjecture, proving that tabular, episodic
reinforcement learning is possible with a sample complexity that
scales only \emph{logarithmically} with the planning horizon.
In other words, when the values are
appropriately normalized (to lie in the unit interval), this results
shows that long
horizon RL is no more difficult than short horizon RL, at least in a minimax sense.

Our analysis introduces two ideas: 
(i) the construction of an $\varepsilon$-net for optimal policies
whose log-covering number scales only logarithmically with the
planning horizon, and
(ii) the Online Trajectory Synthesis algorithm, which 
adaptively evaluates all policies in a given policy class using sample
complexity that scales with the log-covering number of the given
policy class. Both may be of independent interest. 
\end{abstract}

\section{Introduction}
\label{sec:intro}
Long horizons, along with the state dependent transitions, is the differentiator between reinforcement learning (RL) problems and simpler contextual bandit problems.
In the former (RL), actions taken at early stages could substantially impact the future; with regards to planning, the agent must not only consider the immediate reward but also the possible future transitions into differing states.
In contrast, for the latter (contextual bandit problems), the action taken at each time step is independent of the future, though it does depend on the current state (the ``context''); we can consider a contextual bandit problem as a Markov decision process (MDP) with a horizon equal to one.
For a known contextual bandit problem, it is sufficient for the agent to act myopically by choosing the action which maximizes the current reward as a function of the current state.

~\citep{jiang2018open}  proposed to study this distinction by examining how the sample complexity depends on the horizon length (of each episode) in a finite horizon, episodic MDP, where the MDP is unknown to the agent.
Clearly, as the horizon $H$ grows, we will observe more samples in each episode.
To appropriately measure the sample complexity (see~\cite{jiang2018open}), we consider a normalized notion: we
are interested in the \emph{number of episodes} it takes to provably discover a policy whose value is
$\varepsilon$ near to that of the optimal value, where the value is measured
by the \emph{normalized} cumulative reward in each episode (i.e. values are normalized to be bounded between $0$ and $1$).  
Here, all existing upper bounds depend polynomially on the horizon, while
lower bounds do not provide \emph{any} dependence on the horizon $H$. 
Motivated by these observations, \citet{jiang2018open} posed the following open problem in COLT 2018:
\begin{center}
\emph{Can we prove a lower bound that depends polynomially on the planning horizon, $H$?}
\end{center}
\cite{jiang2018open} conjectured a linear dependence on the horizon, which is consistent with all existing upper bounds, which scale at least linearly with the planning horizon~\citep{dann2015sample,azar2017minimax,zanette2019tighter} (see
Section~\ref{sec:rel} for further discussion).
In other words, the conjecture is that, even when the values are
appropriately normalized, long
horizon RL is polynomially more difficult than short horizon RL.

This work resolves this question, with,  perhaps surprisingly, a \emph{negative} answer.
Here we give an informal version of our main result.
\begin{theorem}[Informal version of Theorem~\ref{thm:main}]\label{thm:informal}
	Suppose the reward at each time step is non-negative and the total reward of each episode is upper bounded by $1$.
	Given any target accuracy $0 < \varepsilon < 1$ and a failure probability $0 \le \delta \le 1$, 
	the \emph{Online Trajectory Synthesis} algorithm returns an $\varepsilon$-optimal policy with probability at least $1 - \delta$ by sampling at most $\poly\left(\abs{\states},\abs{\actions},\log H, \frac{1}{\varepsilon},\log\frac{1}{\delta}\right)$
	episodes, where $\abs{\states}$ is the number of states and $\abs{\actions}$ is the number actions.
\end{theorem}
Importantly, this sample complexity scales only \emph{logarithmically} with $H$.
Thus, there does not exist a lower bound that depends polynomially on the planning horizon.
This result is an exponential improvement on the dependency on $H$ over existing upper bounds.

In the context of the discussion in \cite{jiang2018open}, these results suggest that perceived differences between long horizon RL and contextual bandit problems (or short horizon RL) are not attributable to the horizon dependence, at least in a minimax sense. It is worthwhile to note that while our upper bound is logarithmic in $H$, it does have polynomial dependence (beyond just being linear) on the number of states (or ``contexts'') and the number actions. We return to the question of obtaining an optimal rate in Section~\ref{sec:dis}.

\section{Preliminaries}
\label{sec:pre}
Throughout this paper, for a given integer $H$, we use $[H]$ to denote the set $\{1, 2, \ldots, H\}$.
For a condition $\mathcal{E}$, we use $\indict[\mathcal{E}]$ to denote the indicator function, i.e., $\indict[\mathcal{E}] = 1$ if $\mathcal{E}$ holds and $\indict[\mathcal{E}] = 0$ otherwise.

Let $\mdp =\left(\states, \actions, \trans ,R, H, \mu\right)$ be a \emph{Markov Decision Process} (MDP)
where $\states$ is the finite state space, 
$\actions$ is the finite action space, 
$\trans: \states \times \actions \rightarrow \Delta\left(\states\right)$ is the transition operator which takes a state-action pair and returns a distribution over states, 
$R : \states \times \actions \rightarrow \Delta\left( \mathbb{R} \right)$ is the reward distribution,
$H \in \mathbb{Z}_+$ is the planning horizon  (episode length),
and $\mu \in \Delta\left(\states\right)$ is the initial state distribution. 
We refer to a \emph{contextual bandit problem} as an MDP with $H=1$.

A (non-stationary) policy $\pi$ chooses an action $a$ based on the current state $s \in \states$ and the time step $h \in [H]$. 
Formally, $\pi = \{\pi_h\}_{h = 1}^H$ where for each $h \in [H]$, $\pi_h : \states \to \actions$ maps a given state to an action.
The policy $\pi$ induces a (random) trajectory $s_1,a_1,r_1,s_2,a_2,r_2,\ldots,s_{H},a_{H},r_{H}, s_{H + 1}$,
where $s_1 \sim \mu$, $a_1 = \pi_1(s_1)$, $r_1 \sim R(s_1,a_1)$, $s_2 \sim \trans(s_1,a_1)$, $a_2 = \pi_2(s_2)$, $r_2 \sim R(s_2, a_2)$, $\ldots$, $a_H = \pi_H(s_H)$, $r_H \sim R(s_H, a_H)$, $s_{H + 1} \sim \trans(s_H, a_H)$.\footnote{In this paper, for simplicity, we assume the agent receives samples from both $R(s_H, a_H)$ and $\trans(s_H, a_H)$ in the last time step.}

We assume, almost surely,  that $r_h \ge 0$ for all $h \in [H]$ and
\[
\sum_{h = 1}^{H}r_h \in [0, 1] .
\]
In other words, we work with normalized cumulative reward.
It is worth emphasizing that this assumption is weaker than the
standard one in that we do not assume the
immediate rewards $r_h$ are bounded (see Assumptions~\ref{asmp:uniform}
and~\ref{asmp:total_bounded} in Section~\ref{sec:rel} for comparison).
Our goal is to find a policy $\pi$ that maximizes the expected total
reward, i.e.
\[
\max_\pi \expect \left[\sum_{h=1}^{H} r_h \mid \pi\right] .
\]
We say a policy $\pi$ is $\varepsilon$-optimal if $\expect
\left[\sum_{h=1}^{H} r_h\mid \pi\right] \ge \expect
\left[\sum_{h=1}^{H} r_h\mid \pi^*\right] - \varepsilon$, where
$\pi^*$ denotes an optimal policy.

An important concept in RL is the $Q$-function.
Given a policy $\pi$, a level $h \in [H]$ and a state-action pair
$(s,a) \in \states \times \actions$, the $Q$-function is defined as:
\[
Q_h^\pi(s,a) = \expect\left[\sum_{h' = h}^{H}r_{h'}\mid s_h =s, a_h = a, \pi\right].
\]
Similarly, the value function of a given state
$s \in \states$ is defined as: 
\[
V_h^\pi(s)=\expect\left[\sum_{h' = h}^{H}r_{h'}\mid s_h =s,
  \pi\right].
\]
For notational convenience, we denote $Q_h^*(s,a) = Q_h^{\pi^*}(s,a)$
and $V_h^*(s) = V_h^{\pi^*}(s)$.

\section{Related Work}
\label{sec:rel}
We now discuss related theoretical work on tabular RL,
largely focussing on the episodic, finite horizon settings due to
these being most relevant in our setting (also see
\cite{jiang2018open}).   Here, most results focus either on the regret
minimization problem, which aims to collect maximum reward for
a limit number of interactions with the environment, or on the
sample complexity of PAC learning for finding a near-optimal policy. 

In episodic tabular RL,  sample complexities will depend on
$\abs{\states}$, $\abs{\actions}$ and $H$, all of which are assumed to
be finite. In many works, the standard assumption on the rewards are
that: $r_h \in [0,1]$ and hence $\sum_{h=1}^H r_h \in [0, H]$.
However, as pointed out in \cite{jiang2018open}, to have a fair
comparison with contextual bandits (and short horizon RL) and illustrate the hardness due to
the planning horizon, one should scale down by an $H$ factor in order to
normalize the total reward, so that it is bounded in $[0,1]$. 

Following \cite{jiang2018open} (also
see \cite{kakade2003sample}), we can rescale the total reward
(the value functions) to be bounded in $[0,1]$ as follows:
\begin{asmp}[Reward Uniformity, the standard assumption]
\label{asmp:uniform}
The reward received at the $h$-th time step $r_h$ satisfies
$r_h \in [0,1/H]$ for $h=1,\ldots,H$, and hence $\sum_{h=1}^{H}r_h \le 1$.
\end{asmp}
Furthermore, following \cite{jiang2018open}, we can further relax this assumption to
a weaker version where we only bound the total reward as follow.
\begin{asmp}[Bounded Total Reward (see e.g.~\cite{krishnamurthy2016pac})]
	\label{asmp:total_bounded}
	The reward received at the $h$-th time step $r_h$ satisfies
        $r_h \ge 0$ for $h=1,\ldots,H$, and we assume that total
        reward is bounded as $\sum_{h=1}^{H}r_h \le 1$.
\end{asmp}
Assumption~\ref{asmp:total_bounded} is more natural in environments
with sparse rewards, as argued in \cite{jiang2018open}. 
As pointed out in
\cite{jiang2018open}, this scaling permits more fair comparisons with
contextual bandit problems.  Furthermore, Assumption~\ref{asmp:total_bounded} is clearly more
general than Assumption~\ref{asmp:uniform}, so any bound under
Assumption~\ref{asmp:total_bounded}, also implies a bound under Assumption~\ref{asmp:uniform}.

Under Assumption~\ref{asmp:uniform}, a line of work has attempted to provide tight sample complexity bounds~\citep{azar2017minimax,dann2015sample,dann2017unifying,dann2019policy,jin2018q,osband2017posterior}.
To obtain an $\varepsilon$-optimal policy, state-of-the-art results show
that $\widetilde{O}\left(\frac{\abs{\states}\abs{\actions}}{\varepsilon^2} +
  \frac{\poly\left(\abs{\states},\abs{\actions},H\right)}{\varepsilon}\right)$
episodes suffice~\citep{dann2019policy,azar2017minimax}.\footnote{$\widetilde{O}\left(\cdot\right)$ omits logarithmic factors.}
In particular, the first term matches the lower bound $\Omega\left(\frac{\abs{\states}\abs{\actions}}{\varepsilon^2}\right)$ up to logarithmic factors~\citep{dann2015sample,osband2016on}.

There are two concerns with these work.
First, these bound are optimal only in the regime $\varepsilon \in [0,1/H]$.\footnote{In \cite{dann2015sample}, it appears that $\poly(\abs{\states}^2\abs{\actions}/\varepsilon^2)$ episodes are sufficient for $\varepsilon\in[0, 1/H]$. However, it is not clear whether the requirement of $\varepsilon\in[0,1/H]$ can be relaxed to $\varepsilon\in[0, 1]$. 
}
However, as explained in \cite{jiang2018open}, in many scenarios with
a long planning horizon such as control, this regime is not
interesting.\footnote{See the motor control problem described in
  Section 2.1 of \cite{jiang2018open}.} In particular, the more
interesting regime is when $\varepsilon \gg 1/H$.
Secondly, Assumption~\ref{asmp:uniform} is a strong assumption as it cannot model environments with sparse rewards.
\cite{jiang2018open} thus argued that one should study tabular RL with
a more general assumption, i.e. under Assumption~\ref{asmp:total_bounded}.

Note that environments under Assumption~\ref{asmp:total_bounded} can have one-step reward as high as a constant.
These are the sparse reward environments which are often considered to be hard.
\citet{jiang2018open} conjectured that under Assumption~\ref{asmp:total_bounded}, there is a lower bound of the sample complexity that scales \emph{polynomially} in $H$.
Under Assumption~\ref{asmp:total_bounded}, upper bounds in \cite{azar2017minimax,dann2019policy} will become $\widetilde{O}\left(\frac{\abs{\states}\abs{\actions}H^2}{\varepsilon^2}\right)$ whose dependency on $H$ is not tight even in the $\varepsilon \ll 1/H$ regime (recall the lower bound is still $\Omega\left(\frac{\abs{\states}\abs{\actions}}{\varepsilon^2}\right)$).

Recently, under Assumption~\ref{asmp:total_bounded}, \citet{zanette2019tighter} gave a new algorithm which enjoys
a sample complexity of $\widetilde{O}\left(\frac{\abs{\states}\abs{\actions}}{\varepsilon^2} + \frac{\poly\left(\abs{\states},\abs{\actions},H\right)}{\varepsilon}\right)$.\footnote{In fact, \citet{zanette2019tighter} proved a even stronger result that the first term can be a problem-dependent quantity which is upper bounded by $\frac{\abs{\states}\abs{\actions}}{\varepsilon^2} $.}
In particular, the first term matches the lower bound in the regime $\varepsilon \ll 1/H$.
Unfortunately, the second term still scales \emph{polynomially} with $H$.

In another related ``generative model'' setting, where the agent can query samples freely from any state-action pair of the environment, the question of sample complexity is also posed as the total number of batches of queries  (a batch corresponds to $H$ queries) to the environment to obtain an $\varepsilon$-optimal policy.  
Results in this setting include \cite{kearns1999finite, kakade2003sample, singh1994upper, azar2013minimax, sidford2018variance, sidford2018near,agarwal2019optimality}. However, even with this much stronger query model, we are not aware of any algorithm whose sample complexity scales sublinearly in $H$.

The main barrier of achieving logarithmic dependency on $H$ is that almost all the above mentioned works rely on a dynamic programming step (i.e., the Bellman update) in learning the optimal value functions.
In this paper, we bypass this long-standing barrier using a completely different approach.
We provide a new technique to simulate the Monte Carlo methods to evaluate a set of given policies using existing samples and to do exploration.
Interestingly, this technique is related to the method  proposed in
\cite{fonteneau2013batch} for policy evaluation, though there is no
exploration component in their paper. 
Our algorithm has some similarities to the ``Trajectory Tree''  method in \cite{kearns2000approximate} in that both attempt to simultaneously evaluate many policies on a ``tree'' of collected data (using a generative model); a key difference in our approach is that (due to the adaptive nature of data collection) we are not able to explicitly build a tree or reuse data on trajectories.

\section{Main Result and Technical Overview}
\label{sec:main}
\newcommand{\proboverall}{\delta_{\mathsf{overall}}}
The main result of this work follows:
\begin{thm} \label{thm:main}
	Suppose the reward at each level satisfies $r_h \ge 0$ and $\sum_{h=1}^H r_h \le 1$ almost surely. 
	Given a target accuracy $0 < \varepsilon < 1$,
	then with probability at least $1 - \delta$, Algorithm~\ref{alg:main} returns an $\varepsilon$-optimal policy by sampling at most
	\[
	O\left(\frac{|\cS|^3|\cA|^3\log^2 H}{\varepsilon^3} \log\left(\frac{|\cS||\cA|}{\varepsilon}\right) \cdot \left(|\states|^2|\actions|\log\left(\frac{H|\states|}{\varepsilon}\right) + \log\left(\frac{1}{\delta}\right)\right)
	\right)
	\]
	episodes.
\end{thm}
This result shows that only \emph{logarithmic} dependence on the horizon is possible.
Algorithm~\ref{alg:main} is provided in Section~\ref{sec:proof}.
We remark that, while our bound improves the dependency on $H$, the dependency on $\abs{\states}$, $\abs{\actions}$ and $1 / \varepsilon$ are worse than existing state-of-the-art bounds (cf. Section~\ref{sec:rel}). It is an open problem to further tighten the dependencies on $\abs{\states}$, $\abs{\actions}$ and $\frac{1}{\varepsilon}$. 

We now provide an overview of our analysis.
\subsection{Technical Overview}
\label{sec:techover}
\paragraph{An $\varepsilon$-net For Non-stationary Policies.}
We first construct a set of polices $\Pi$ which contains an $\varepsilon$-optimal policy for any MDP.
Importantly, the size of $\Pi$ satisfies $|\Pi| = (H / \varepsilon)^{\mathrm{\poly(|\states||\actions|)}}$, which is acceptable since the overall sample complexity of our algorithm depends only logarithmically on $|\Pi|$.
To define such a set of policies, we consider all discretized MDPs whose transition probabilities and reward values are integer multiples of $\mathrm{poly}(\varepsilon/(|\states||\actions|H))$.
Clearly, there are most $(H / \varepsilon)^{\mathrm{\poly(|\states||\actions|)}}$ such discretized MDPs, and for each discretized MDP $\mdp$, we add an optimal policy of $\mdp$ into $\Pi$.
It remains to show that for any $\mdp$, there exists a policy $\pi \in \Pi$ which is an $\varepsilon$-optimal policy of $\mdp$.
This can be seen since there exists a discretized MDP $\hat{\mdp}$ whose transition probabilities and reward values are close enough to those of $\mdp$, and by standard perturbation analysis, it can be easily shown that an optimal policy of $\hat{M}$ is an $\varepsilon$-optimal policy of $\mdp$.
The formal analysis is given in Section~\ref{sec:policy}.

\paragraph{The Trajectory Synthesis Method.}
Now we show how to evaluate values of all policies in the policy set $\Pi$ constructed above by sampling at most $\mathrm{poly}(|\states|, |\actions|, 1 / \varepsilon, \log |\Pi|, \log H)$ episodes. 
To achieve this goal, we design a trajectory simulator, which, for every policy in the set, either interacts with the environment to collect trajectories, or simulates trajectories using collected samples. 
In either case, the simulator obtains trajectories of the policy with distribution close enough to those sampled by interacting with the environment.
The most natural idea is to collect trajectories for each policy $\pi$ separately by interacting with the environment. 
This method, although is guaranteed to output ``true'' trajectories for every policy, has sample complexity at least linear in the size of the policy set $|\Pi|$ and is thus insufficient for our goal.
Another possible way to evaluate policies is to build an empirical model (an estimation of transition probability and reward function) and evaluate policies on the empirical model (or to build a trajectory tree as in~\cite{kearns2000approximate}).
However, we do not know how to deal with the dependency issue in building the empirical model and to prove a sample complexity bound that scales logarithmically with the  planning horizon.
The analysis based on performance difference lemma can lead to polynomial dependency on the planning horizon~\citep{kakade2003sample}.

\paragraph{Reuse Samples.} A key observation is that once we obtain a trajectory for a policy by interacting with the environment, samples collected during this process can be used to simulate trajectories for other policies.
To better illustrate this idea, we use $\Pi_\cD$ to denote the set of policies for which we have obtained trajectories by interacting with the environment, and denote 
\[
\cD_{s,a}=\left[\left(s_{(s,a)}^{(1)}, r_{(s,a)}^{(1)}\right), \left(s_{(s,a)}^{(2)}, r_{(s,a)}^{(2)}\right), \ldots\right]
\]
to be the sequence of samples obtained from $\trans(s, a)$ and $R(s, a)$.
These samples are sorted in chronological order.
Suppose that now we are given a new policy $\pi$ and for all $(s, a) \in \cS \times \cA$, $|\cD_{s,a}^{(t)}|\ge H$.
Then it is easy to simulate a trajectory for $\pi$ using the set of samples $\left\{ \cD_{s,a}\right\}_{s \in \states, a \in \actions}$.
Indeed, we start from state $s_1$ and set $(s_2, r_2)$ to be the first pair in $\cD_{s_1,\pi_1(s_1)}$, and then set $(s_3, r_3)$ to be the first pair in $\cD_{s_2,\pi_2(s_2)}$ that has not been used, etc.
In general, suppose we are at state $s_{h}$ for some $h < H$, we set $(s_{h + 1}, r_{h + 1})$ to be the first pair in $\cD_{s_h, \pi_h(s_h)}$ that has not been used.
Note that such a procedure generates a trajectory for $\pi$ with exactly the same distribution as that generated by interacting with the environment.

\paragraph{Avoid Unnecessary Sampling.}
We have described the approach to reuse samples in the above paragraph. 
Nevertheless, there is a problem intrinsic to the above approach: if the process of simulating a policy $\pi$ fails (i.e., some $(s_h,\pi_h(s))$ has been visited $j\le H$ times but $|\cD_{s_h,\pi_h(s)}|<j$), should we interact with the environment to generate a trajectory or simply claim failure?
Note that claiming failure is acceptable as long as the overall failure probability is small.

In order to decide when to interact with the environment, we design a procedure to estimate the probability of simulation failure.
If the failure probability is already small enough, there is no need to interact with the environment.
Otherwise, we interact with the environment to obtain a trajectory. 
To bound the overall sample complexity, one key observation is that if the failure probability is large, then the policy will visit some state-action pair more frequently than all existing policies.
In Section~\ref{sec:analysis}, we make this intuition rigorous by designing a potential function to measure the overall progress made by our algorithm.

\section{The Proof}
\label{sec:proof}
In this section, for the sake of presentation, we assume a fixed initial state $s_1$.
When the initial	state is sampled from a distribution $\mu$, we may create a new state $s_0$ and set $s_0$ to be the initial state.
We set $\trans(s_0, a) = \mu$ and $r(s_0, a) = 0$ for all $a \in \actions$, and increase the planning horizon $H$ by $1$.
By doing so, now $s_1$ is sampled from the initial state distribution $\mu$.

\subsection{An $\varepsilon$-net For Non-stationary Policies}
\label{sec:policy}
In this section, we construct a set of polices which contains a near-optimal policy for any MDP.
To define these policies, we first define a set of MDPs.

Throughout this section, without loss of genearlity, we assume $1 / \varepsilon$ is a positive integer. 
In general, we may decrease $\varepsilon$ by a factor of at most two so that $1 / \varepsilon$ is a positive integer. 

The following definition is helpful in our analysis.
\begin{definition}
For an MDP $\mdp =\left(\states, \actions, \trans ,R, H, \mu \right)$, we say a pair $(s, h) \in \states \times [H]$ is {\em admissible} with respect to $\mdp$ if there exists a policy $\pi$ such that
$\Pr[s_h = s \mid \pi] > 0$.
\end{definition} 

Before presenting our analysis, we prove the following property regarding admissible pairs.

\begin{lemma}\label{lem:admissible}
For any admissible $(s, h) \in \states \times [H]$, for any $a \in \actions$, the following hold:
\begin{itemize}
\item $0 \le R(s, a) \le 1$ almost surely;
\item $0 \le Q^{\pi}_h(s, a) \le 1$ for any policy $\pi$;
\item $0 \le V^{\pi}_h(s) \le 1$ for any policy $\pi$.
\end{itemize}
\end{lemma}
\begin{proof}
Here we only prove $0 \le R(s, a) \le 1$. It can be similarly proved that $0 \le Q^{\pi}_h(s, a) \le 1$ and $0 \le V^{\pi}_h(s) \le 1$.
Suppose $R(s, a) > 1$ or $R(s, a) < 0$ with non-zero probability. 
Since  $(s, h)$ is admissible, there exists a policy $\pi$ such that $\Pr[s_h = s \mid \pi] > 0$.
Consider the policy $\pi'$ defined to be:
\[
\pi'_{h'}(s) = \begin{cases}
\pi_{h'}(s) & h' < h\\
a& h' \ge  h
\end{cases}.
\]
Clearly, $r_h > 1$ or $r_h < 0$ with non-zero probability, which violates the assumption that $\sum_{h = 1}^{H}r_h \in [0, 1]$ and $r_h \ge 0$ for all $h \in [H]$ almost surely.
\end{proof}

\begin{definition}[Discretized MDPs]
For given $\states$, $\actions$, $H$, $s_1$ and $\varepsilon > 0$, define $\mdps_{\varepsilon}$ to be the set of MDPs $\mdp =\left(\states, \actions, \trans ,R, H, s_1\right)$ such that
\begin{itemize}
\item Rewards are deterministic and for any $(s, a) \in \states \times \actions$, $R(s, a) \in \{0, \varepsilon, 2 \varepsilon, 3\varepsilon, \ldots, 1\}$;
\item For each $(s, a, s') \in \states \times \actions \times \states$, $\trans(s' \mid s, a)  \in \{0, \varepsilon, 2 \varepsilon, 3\varepsilon, \ldots, 1\}$;
\end{itemize}
\end{definition}

The following lemma gives an upper bound on the size of $\mdps_{\varepsilon}$.
\begin{lemma}\label{lem:policy_count}
$\left|\mdps_{\varepsilon}\right| \le (1 / \varepsilon + 1)^{|\states|^2 |\actions| + |\states| |\actions|}$.
\end{lemma}
\begin{proof}
Since each $\mdp \in \mdps_{\varepsilon}$ is uniquely defined by its $R$ and $\trans$, below we count the number of possible $R$ and $\trans$ respectively.

Since rewards are deterministic and for any $(s, a) \in \states \times \actions$, $R(s, a) \in \{0, \varepsilon, 2 \varepsilon, \ldots, 1\}$, there are $(1 / \varepsilon + 1)^{|\states||\actions|}$ different rewards in total.

Since for each $(s, a, s') \in \states \times \actions \times \states$, $\trans(s' \mid s, a)  \in \{0, \varepsilon, 2 \varepsilon, \ldots, 1\}$, there are at most $(1 / \varepsilon + 1)^{|\states|^2|\actions|}$ different transitions in total.

Therefore, $\left|\mdps_{\varepsilon}\right| \le (1 / \varepsilon + 1)^{|\states|^2 |\actions| + |\states| |\actions|}$.

\begin{definition}[$\varepsilon$-net for Non-stationary Policies]
	\label{defn:discrete_policies}
For given $\states$, $\actions$, $H$ and $\varepsilon > 0$, define $\Pi_{\varepsilon}$ to be the set of polices such that
\[
\Pi_{\varepsilon} = \{\pi_{\mdp} \mid \text{$\pi_{\mdp}$ is an optimal policy for $\mdp \in \mdps_{\varepsilon}$}\}.
\]
For each $\mdp \in \mdps_{\varepsilon}$, when $\mdp$ has multiple optimal policies, we add an arbitrary one to $\Pi_{\varepsilon}$.
\end{definition}
By construction of $\Pi_{\varepsilon}$ and Lemma~\ref{lem:policy_count}, it is clear that $\left|\Pi_{\varepsilon} \right|  \le (1 / \varepsilon + 1)^{|\states|^2 |\actions| + |\states| |\actions|}$
\end{proof}

Now we prove that for any MDP $\mdp$, there is a near-optimal policy $\pi \in \Pi_{\varepsilon}$.

\begin{lemma}\label{lem:discrete_near_optimal}
For any MDP $\mdp =\left(\states, \actions, \trans ,R, H, s_1\right)$, there exists $\pi \in \Pi_{\varepsilon}$ such that $\pi$ is $8H|\states|\varepsilon$-optimal.
\end{lemma}
\begin{proof}
We first show that there exists $\hat{\mdp} =\left(\states, \actions, \hat{\trans} ,\hat{R}, H, s_1\right) \in \mdps_{\varepsilon}$ such that the following hold:
\begin{itemize}
\item For any $(s, h) \in \states \times [H]$ admissible with respect to $\mdp$, for any $a \in \actions$, $|\hat{R}(s, a) - \expect[R(s, a)]| \le \varepsilon$;
\item For each $(s, a, s') \in \states \times \actions \times \states$, $\left|\trans(s' \mid s, a) - \hat{\trans}(s' \mid s, a)\right| \le \varepsilon$;
\item For each $(s, a, s') \in \states \times \actions \times \states$, if $\trans(s' \mid s, a)= 0$ then $\hat{\trans}(s' \mid s, a) = 0$;
\end{itemize}

Below we construct such $\hat{\trans}$ and $\hat{R}$.
By Lemma~\ref{lem:admissible} we have $\expect[R(s, a)] \in [0, 1]$.
Therefore, by setting $\hat{R}(s, a)$ to be closest real number in $\{0, \varepsilon, 2\varepsilon, \ldots, 1\}$, we have $|\hat{R}(s, a) - \expect[R(s, a)]| \le \varepsilon$.
Furthermore, for each $(s, a, s') \in \states \times \actions \times \states$, we set 
\[
\trans'(s' \mid s, a) = \min \{p \in \{0, \varepsilon, 2\varepsilon, \ldots, 1\} \mid p \ge\trans(s' \mid s, a) \}.
\]
Notice that $P'(s, a)$ may not always be a probability distribution.
Clearly $\trans'(s' \mid s, a) \ge \trans(s' \mid s, a)$ for each $(s, a, s') \in \states \times \actions \times \states$ and $\sum_{s' \in \states} P'(s' \mid s, a)  = 1 + k\varepsilon$ for some positive integer $0 \le k \le |\states|$.
Now for each $(s, a)$, we set $\hat{\trans}(s' \mid s, a)= \trans'(s' \mid s, a) - \varepsilon$ for an arbitrary $k$ states $s' \in \states$ with $\trans(s' \mid s, a) > 0$, and set  $\hat{\trans}(s' \mid s, a)= \trans'(s' \mid s, a)$ for all other states $s'$.
Clearly, $P'(s, a)$ is a probability distrbution for any $(s, a)$ and satisfies the desired property. 

Now for any policy $\pi$, we use $V^{\pi}$ to denote the $V$-value of $\pi$ on MDP $\mdp$, and use $\hat{V}^{\pi}$ to denote the $V$-value of $\pi$ on $\hat{\mdp}$. $Q^{\pi}$ and $\hat{Q}^{\pi}$ are defined analogously. We prove that $|V^{\pi} - \hat{V}^{\pi}| \le 4|\states|H\varepsilon$ for any policy $\pi$ inductively by the following induction hypothesis:
\begin{itemize}
\item $|V_h^{\pi}(s) - \hat{V}_h^{\pi}(s)| \le (1 + (H - h)(|\states| + 1))\varepsilon$ for any admissible $(s, h)$;
\item $|Q_h^{\pi}(s, a) - \hat{Q}_h^{\pi}(s, a)| \le  (1 + (H - h)(|\states| + 1))\varepsilon$ for any admissible $(s, h)$ and any $a \in \actions$.
\end{itemize}

When $h = H$, $V_H^{\pi}(s) = Q_H(s, \pi_H(s)) = \expect[R(s, \pi_H(s))]$ and $\hat{V}_H^{\pi}(s) = \hat{Q}_H(s, \pi_H(s)) = \hat{R}(s, \pi_H(s))$.
Therefore, the induction hypothesis holds when $h = H$ since $|\hat{R}(s, a) - \expect[R(s, a)]| \le \varepsilon$.

Now we show the induction hypothesis holds for any $h < H$.
For any $h < H$, consider any state $s$ such that $(s, h)$ is admissible.
Notice that $V_h^{\pi}(s) = Q_h(s, \pi_h(s))$ and  $\hat{V}_h^{\pi}(s) = \hat{Q}_h(s, \pi_h(s))$, and therefore $|V_h^{\pi}(s) - \hat{V}_h^{\pi}(s)|  = |Q_h^{\pi}(s, \pi_h(s)) - \hat{Q}_h^{\pi}(s, \pi_h(s))|$.
Furthermore, 
\[
Q_h^{\pi}(s, a) = \expect[R(s, a)] + \sum_{s' \in \states} \trans(s' \mid s, a) V_{h + 1}^{\pi}(s') 
\]
and
\[\hat{Q}_h^{\pi}(s, a) = \hat{R}(s, a) + \sum_{s' \in \states} \hat{\trans}(s' \mid s, a)V_{h + 1}^{\pi}(s').\]
Therefore,
\begin{align*}
&\left|Q_h^{\pi}(s, a) - \hat{Q}_h^{\pi}(s, a)\right| \\
\le &  \left| \expect[R(s, a)]  - \hat{R}(s, a)\right| + \sum_{s' : \trans(s' \mid s, a) > 0} \left|\trans(s' \mid s, a) V_{h + 1}^{\pi}(s') - \hat{\trans}(s' \mid s, a)\hat{V}_{h + 1}^{\pi}(s')\right|\\
\le &  \varepsilon + \sum_{s' : \trans(s' \mid s, a) > 0} \left( \left|\trans(s' \mid s, a) - \hat{\trans}(s' \mid s, a)\right| \cdot V_{h + 1}^{\pi}(s') + \hat{\trans}(s' \mid s, a) \cdot \left| V_{h + 1}^{\pi}(s' ) - \hat{V}_{h + 1}^{\pi}(s')\right| \right)\\
\le &  (|\states| + 1)\varepsilon + \sum_{s' : \trans(s' \mid s, a) > 0}  \hat{\trans}(s' \mid s, a) \cdot \left| V_{h + 1}^{\pi}(s' ) - \hat{V}_{h + 1}^{\pi}(s')\right| \tag{$V_{h + 1}^{\pi}(s') \le 1$ by Lemma~\ref{lem:admissible}} \\
\le &  (|\states| + 1)\varepsilon + (1 + (H - (h + 1))(|\states| + 1))\varepsilon \tag{$\sum_{s' \in \states}  \hat{\trans}(s' \mid s, a)= 1$ and induction hypothesis} \\
= &  (1 + (H - h)(|\states| + 1))\varepsilon.
\end{align*}
Thus, we have
\[
\left|V^{\pi}_1(s_1) - \hat{V}^{\pi}_1(s_1)\right| \le 4|\states|H \varepsilon.
\]

Finally, consider any optimal policy $\hat{\pi}$ of $\hat{\mdp}$ and any optimal policy $\pi$ of $\mdp$, we have
\[
V^{\hat{\pi}}_1(s_1) \ge \hat{V}^{\hat{\pi}}_1(s_1) -  4|\states|H \varepsilon \ge \hat{V}^{\pi}_1(s_1) -  4|\states|H \varepsilon \ge V^{\pi}_1(s_1) - 8|\states|H\varepsilon.
\]
Since $\hat{\pi} \in \Pi_{\varepsilon}$, the lemma holds.
\end{proof} 
\subsection{Evaluating Policies}
\label{sec:evaluate}
\newcommand{\probsim}{\delta_{\mathsf{sim}}}

As shown in Section~\ref{sec:policy}, there exists a set of policies $\Pi$ such that for any MDP $\mdp$, there exists a near-optimal policy $\pi \in \Pi$.
In this section, we show how to approximately evaluate the values of all policies in $\Pi$ using at most $\mathrm{poly}(|\states|, |\actions|, 1 / \varepsilon, \log |\Pi|, \log H)$ episodes. 
We formally describe our simulator in Section~\ref{sec:simulator} and present its analysis in Section~\ref{sec:analysis}.

\subsubsection{The Trajectory Simulator}\label{sec:simulator}
In this section, we describe our algorithm for simulating trajectories.
The algorithm is formally presented in Algorithm~\ref{alg:tsimall} and Algorithm~\ref{alg:tsim}.
Algorithm~\ref{alg:tsim} receives a parameter $\tau$ and uses a replay buffer $\cD$ to store samples.
Formally, $\cD=\{\cD_{s,a}\}_{s \in \states, a \in \actions}$, where each $\cD_{s,a}$ contains samples associated with state-action pair $(s,a)$, i.e., \[\cD_{s,a}=\left[(s^{(1)}_{s,a}, r_{s,a}^{(1)}), (s^{(2)}_{s,a}, r_{s,a}^{(2)}), \ldots \right]\] and samples are sorted in chronological order.
  We also maintain $\Pi_{\mathcal{D}}$ in Algorithm~\ref{alg:tsim} which is the set of policies used to generate $\cD$.
There are two subroutines in Algorithm~\ref{alg:tsim}.
Subrountine \textsc{Simulate} takes an input policy $\pi$ and outputs either \texttt{Fail} or a trajectory for policy $\pi$.
Subroutine \textsc{Rollout} takes an input policy $\pi$, samples $\tau$ episodes for $\pi$ by interacting with the environment and stores all collected samples in the replay buffer $\cD$.
It also returns one of the $\tau$ trajectories sampled for for $\pi$.
Moreover, whenever Subroutine \textsc{Rollout} is invoked, samples in $\cD$ are recollected so that independence among samples in the replay buffer $\cD$ is ensured. 

Algorithm~\ref{alg:tsimall} receives a failure probability $\probsim$ and a policy set $\Pi$ as inputs.
In Algorithm~\ref{alg:tsimall}, we run $F$ independent copies of  Algorithm~\ref{alg:tsim} in parallel. 
For each policy $\pi$, for the $F$ independent copies of  Algorithm~\ref{alg:tsim}, Algorithm~\ref{alg:tsimall} checks whether Subroutine \textsc{Simulate} returns \texttt{Fail} for too many times.
If so, it calls Subroutine \textsc{Rollout} for each copy of Algorithm~\ref{alg:tsim} to collect samples and produce trajectories for $\pi$.
Otherwise, it directly returns trajectories returned by Subroutine \textsc{Simulate}.
The formal analysis of our algorithms will presented in Section~\ref{sec:analysis}.

\begin{algorithm}[t]
  \caption{\texttt{SimAll}\label{alg:tsimall}}
  \begin{algorithmic}[1]
  \State \textbf{Input:} failure probability $\probsim$, policy set $\Pi$, number of trajectories $F$ 
  \State $\tau \gets 16|\states|\abs{\actions} / \probsim \cdot \log(4 |\states|\abs{\cA} / \probsim)$
  \For{$i \in [F]$} \Comment{Run $F$ copies of Algorithm~\ref{alg:tsim} in parallel}
  \State Set $\mathsf{SO}_i$ to be the $i$-th independent copy of \texttt{SimOne}$(\tau)$ (Algorithm~\ref{alg:tsim})
  \EndFor
  \For{$\pi \in \Pi$}
  \For{$i \in [F]$}
  \State $z_i^{\pi} \gets \mathsf{SO}_i$.\textsc{Simulate}$(\pi)$
  \EndFor
  \If{$\sum_{i = 1}^F \indict[z_i^{\pi} \text{ is \texttt{Fail}}]> 3\probsim/2 \cdot F$}\label{algline:simfail}
  \For{$i \in [F]$}
  \State $z_i^{\pi} \gets \mathsf{SO}_i$.\textsc{Rollout}$(\pi)$
  \EndFor
  \EndIf
\EndFor
  \State \Return{$\{z_i^{\pi}\}_{(i, \pi) \in [F] \times \Pi}$}
  \end{algorithmic}
\end{algorithm}

\begin{algorithm}[t]
  \caption{\texttt{SimOne}\label{alg:tsim}}
  \begin{algorithmic}[1] 
  \State \textbf{Input:} number of repetitions $\tau$
  \Function{Simulate}{$\pi$}: 
     \For{$(s, a) \in \states \times \actions$}
     \State Mark all elements in $\cD_{s, a}$ as \texttt{unused}
     \EndFor
      \For{$h \in \{1, 2, \ldots, H\}$}
      \If{all elements in $\cD_{s_h, \pi_h(s_h)}$ are marked as \texttt{used}}
      \State \Return \texttt{Fail}
      \Else
      \State Set $(s_{h + 1}, r_{h})$ to be the first element in $\cD_{s_h, \pi_h(s_h)}$ that is marked as \texttt{unused}
      \State Mark $(s_{h + 1}, r_{h})$ (the first unused element in $\cD_{s_h, \pi_h(s_h)}$) as \texttt{used}
      \EndIf
     \EndFor
     \State \Return $(s_1, \pi_1(s_1), r_1), (s_2, \pi_2(s_2), r_2), \ldots, (s_H, \pi_H(s_H), r_H)$
  \EndFunction
  \Function{Rollout}{$\pi$}
    \State Set $\cD_{s,a}$ to be an empty sequence for all $(s, a) \in \states \times \actions$ 
    \State $\Pi_{\mathcal{D}}\gets \Pi_{\mathcal{D}}\cup \{\pi\}$ \label{line:add_policy}
    \For{$\pi'\in \Pi_{\mathcal{D}}$}
    	\State Sample $\tau$ trajectories for $\pi'$ by interacting with the environment \label{line:rollout}
	\State Add all collected samples to $\cD$
    \EndFor
    \State \Return one of the $\tau$ trajectories sampled for $\pi$
  \EndFunction
\end{algorithmic}
\end{algorithm}

\subsubsection{Analysis}\label{sec:analysis}
In this section, we present the formal analysis of Algorithm~\ref{alg:tsimall} and Algorithm~\ref{alg:tsim}.
Before we present our analysis, we first introduce some necessary notations.
\begin{defn}
For any policy $\pi$, for any state-action pair $(s, a) \in \states \times \actions$, define $f^{\pi}(s,a) \in [H]$ to be a random variable defined as
\[
f^{\pi}(s,a) = \sum_{h = 1}^H \indict[(s, a) = (s_h, a_h) \mid \pi].
\]
I.e., $f^{\pi}(s,a)$ is the random variable which is the total number of times a trajectory induced by $\pi$ visits $(s,a)$.
\end{defn}

We additionally define the following quantity to characterize the number times a state-action pair is visited by a set of policies.
Intuitively, given a success probability $\delta$, this quantity measures the maximum number of times a policy within a given policy set can visit a particular $(s,a)$ pair.
\begin{defn}
	\label{defn:policy-freq}
For a set of policies $\Pi$, for any $(s,a)\in \states \times \actions$, define
\[
 \mu^{\Pi}_{\delta}(s,a) = \max\big\{\lambda \mid \lambda\in [0,H], \max_{\pi\in \Pi}\Pr[f^{\pi}(s,a) \ge \lambda]\ge \delta\big\}.
\]
\end{defn}
Note that $\mu^{\Pi}_{\delta}(s,a)$ is always a non-negative integer since for any state-action pair $(s, a) \in \states \times \actions$, policy $\pi$ and real number $\lambda$, 
\[
\Pr\left[f^{\pi}(s,a) \ge \lambda\right] = \Pr\left[f^{\pi}(s,a) \ge \lceil{\lambda\rceil}\right].
\]

Our next lemma states that for some policy $\pi$, if \texttt{SimOne} fails with high probability, then there exists a state-action pair that $\pi$ visits more frequently than all previous policies.
\begin{lemma}
\label{lem:prob-increase}
For a policy $\pi\in \Pi$, suppose Subroutine \textsc{Simulate} in Algorithm~\ref{alg:tsim} returns \texttt{Fail} with probability at least $\probsim$ over the randomness of the generating process of the replay buffer $\cD$.
There exists $(s, a) \in \states \times \actions$ such that
\[
\Pr\Big[f^{\pi}(s,a) > \tau\cdot\frac{\probsim}{4|\cS|\abs{\cA}}\cdot  \mu_{\probsim/(2|\cS|\abs{\cA})}^{\Pi_{\cD}}(s,a)\Big]\ge \frac{\probsim}{2|\cS|\abs{\cA}}
\]
where $\Pi_{\mathcal{D}}$ is the set of policies used to generate $\cD$.
\end{lemma}
\begin{proof}
Suppose for the sake of contradiction that for each $(s, a) \in \states \times \actions$,
\[
\Pr\Big[f^{\pi}(s,a) > \tau\cdot\frac{\probsim}{4|\cS|\abs{\cA}}\cdot  \mu_{\probsim/(2|\cS|\abs{\cA})}^{\Pi_{\cD}}(s,a)\Big]< \frac{\probsim}{2|\cS|\abs{\cA}}.
\]
Let us denote 
\[
\Gamma(s,a) = \tau\cdot\frac{\probsim}{4|\cS|\abs{\cA}}\cdot  \mu_{\probsim/ (2|\cS|\abs{\cA})}^{\Pi_\cD}(s,a).
\]
For each $(s,a) \in \states \times \actions$, we have
\[
\Pr\Big[f^{\pi}(s,a) >\Gamma(s,a) \Big]< \frac{\probsim}{2|\cS|\abs{\cA}}.
\]
Therefore, by a union bound over all state-action pairs $(s, a) \in \cS \times \actions$, with probability at least $1 - \probsim/2$, for all $(s,a) \in \states \times \actions$,
\begin{align}
 f^{\pi}(s,a) \le \Gamma(s,a). \label{eqn:smaller_than_thres}
\end{align}

For each $(s,a)\in \states \times \actions$, define $\cE_{s,a}$ to be the event that
\[
\cE_{s,a} = \left\{\left| \cD_{s, a}\right| \ge  \Gamma(s, a)\right\}.
\]
By Definition~\ref{defn:policy-freq}, there exists a policy $\pi_{s,a}^* \in \Pi_{\cD}$ such that
\[\Pr\left[f^{\pi_{s,a}^*}(s,a) \ge  \mu^{\Pi_\cD}_{\probsim / (2|\states|\abs{\actions})}(s,a)\right]\ge \probsim / (2 |\states|\abs{\cA}).\]
Now consider Line~\ref{line:rollout} in Subroutine~\textsc{Rollout} in Algortihm~\ref{alg:tsim}.
Define
\[
X_i = \begin{cases}
1 & \text{if $\sum_{h = 1}^H \indict[(s_h, a_h) = (s, a)] \ge  \mu^{\Pi_\cD}_{\probsim / (2|\states|\abs{\actions})}(s,a)$ for the $i$-th trajectory of $\pi^*_{s, a}$}\\
0 & \text{otherwise}
\end{cases}.
\]
Note $X_1,\ldots,X_\tau$ are i.i.d. random variables.
By definition, $\expect[X_i] \ge \probsim / (2 |\states|\abs{\actions})$.
Therefore, since $\tau = 16|\states|\abs{\actions} / \probsim \cdot \log(4 |\states|\abs{\actions} / \probsim)$, by Chernoff bound,
\[
\Pr\left[\sum_{i=1}^{\tau} X_i \le \tau \cdot \frac{\probsim}{4|\cS|\abs{\actions}}\right]
\le \exp\left(-\frac{\tau \probsim/(2|\cS|\abs{\cA})}{8}\right)
\le \frac{\probsim}{4|\cS|\abs{\cA}}.
\]

Therefore,
\[
\Pr[\cE_{s,a}] \ge\Pr\left[\sum_{i=1}^{\tau} X_i \ge \tau \cdot \frac{\probsim}{4|\cS|\abs{\cA}}\right] \ge 1 - \frac{\probsim}{4 |\states|\abs{\cA}}.
\]
It follows that with probability at least $1 - \probsim / 4$, for all $(s, a) \in \states\times\actions$,
\begin{align}
\left| \cD_{s, a}\right| \ge  \Gamma(s, a). \label{eqn:thres_upper}
\end{align}

By a union bound over~\eqref{eqn:smaller_than_thres} and~\eqref{eqn:thres_upper}, with probability at least $1 - \frac{3\probsim}{4}$, for all $(s,a) \in \states \times \actions$,
\[
\left|\cD_{s, a}\right| \ge  \Gamma(s, a) \ge  f^{\pi}(s,a),
\]
in which case Subroutine \textsc{Rollout} does not return \texttt{Fail}.
This contradicts the assumption that Subroutine \textsc{Rollout} returns \texttt{Fail} with probability at least $\probsim$.

\end{proof}
Now we discuss the implication of Lemma~\ref{lem:prob-increase}.
Note that Algorithm~\ref{alg:tsim} interacts with the environment to sample trajectories only when Subroutine~\textsc{Simulate} fails with probability at least $\probsim$.
By Lemma~\ref{lem:prob-increase}, when Algorithm~\ref{alg:tsim} interacts with the environment to sample trajectories, $\mu_{\probsim/(2|\cS|\abs{\actions})}^{\Pi_\cD}(s,a)$ doubles or changes from $0$ to $1$ for some $(s, a) \in \states \times \actions$ since $\tau \probsim / (4|\cS|\abs{\actions}) > 2$.
However, $\mu_{\probsim/(2|\cS|\abs{\actions})}^{\Pi_\cD}(s,a)$ is always upper bounded by $H$.
Therefore, the total number of calls to Subrountine \textsc{Rollout} in Algorithm~\ref{alg:tsimall} is upper bounded by $O(|\cS||\cA|\log H)$.
Our next lemma guarantees that whenever Algorithm~\ref{alg:tsimall} invokes Subroutine \textsc{Rollout}, the probability that Subroutine \textsc{Simulate} returns \texttt{Fail} is at least $\probsim$, and when Subroutine~\textsc{Rollout} is not invoked, the probability that Subroutine \textsc{Simulate} returns \texttt{Fail} is at most $2\probsim$.
\begin{lemma}
\label{lem:chern}
Suppose $F \ge 24 / \probsim \cdot \log (2 |\Pi| / \proboverall)$.
With probability at least $1-\proboverall / (2|\Pi|)$, 
each time Line~\ref{algline:simfail} in Algorithm~\ref{alg:tsimall} is executed, 
the following hold:
\begin{itemize}
    \item when $\sum_{i = 1}^F \indict[z_i^{\pi} \text{ is \texttt{Fail}}] > 3\probsim/2 \cdot F$, the probability that Subroutine \textsc{Simulate} returns \texttt{Fail} is at least $\probsim$ over the randomness of the generating process of the replay buffer $\cD$;    
    \item when $\sum_{i = 1}^F \indict[z_i^{\pi} \text{ is \texttt{Fail}}] \le 3\probsim/2 \cdot F$, the probability that Subroutine \textsc{Simulate} returns \texttt{Fail} is at most $2\probsim$ over the randomness of the generating process of the replay buffer $\cD$.
    \end{itemize}
\end{lemma}
\begin{proof}
Let $Y_i = \indict[\text{$z_i^{\pi}$ is \texttt{Fail}}]$.
Note that each time Subroutine \textsc{Rollout} is invoked, all samples in $\cD$ are recollected. 
Therefore, for any given time step of the algorithm, $\{Y_i\}_{i = 1}^{F}$ are independent random variables.

If $\Pr[Y_i = 1]<\probsim$, by Chernoff bound,
\[
\Pr\left[\sum_{i=1}^{F}Y_i \ge 3\probsim/2 \cdot F\right] \le \exp(-\probsim F / 24) \le \frac{\proboverall}{2|\Pi|}.
\]

On the other hand, if $\Pr[Y_i = 1]\ge 2\probsim$, by Chernoff bound,
\[
\Pr\left[\sum_{i=1}^{F}Y_i \le 3\probsim/2 \cdot F\right] \le \exp(-\probsim F / 16)\le \frac{\proboverall}{2|\Pi|}.
\]
Thus the lemma holds.
\end{proof}
Now we bound the overall sample complexity of the algorithm. 
\begin{lemma}
\label{lem:double}
Suppose $F \ge 24 / \probsim \cdot \log (2 |\Pi| / \proboverall)$.
Let $\Pi_{\cD}$ be the set of policies maintained by Algorithm~\ref{alg:tsim} before executing Line~\ref{line:add_policy}, and let $\hat{\Pi}_{\cD}$ be the set of policies maintained after executing Line~\ref{line:add_policy}, i.e., $\hat{\Pi}_{\cD} = \Pi_{\cD} \cup \{\pi\}$.
With probability at least $1-\proboverall/(2|\Pi|)$, there exists $(s,a) \in \states \times \actions$, such that
\[
\mu_{\probsim/{(2|\cS||\actions|)}}^{\hat{\Pi}_{\cD}}(s,a) \ge \max \left(
2\cdot\mu_{\probsim/{(2|\cS|\abs{\actions})}}^{\Pi_{\cD}}(s,a),  1
\right).
\]
\end{lemma}
\begin{proof}
By Lemma~\ref{lem:chern}, with probability at least $1-\proboverall/(2|\Pi|)$, 
for the added policy $\pi$, 
the probability that Subroutine \textsc{Simulate} returns \texttt{Fail} is at least $\probsim$.
By Lemma~\ref{lem:prob-increase}, 
there exists $(s, a) \in \states \times \actions$ such that
\[
\Pr\Big[f^{\pi}(s,a) > \tau\cdot\frac{\probsim}{4|\cS|\abs{\actions}}\cdot  \mu_{\probsim/(2|\cS|\abs{\actions})}^{\Pi_{\cD}}(s,a)\Big]\ge \frac{\probsim}{2|\cS|\abs{\actions}}.
\]

If $\mu_{\probsim/(2|\cS|\abs{\actions})}^{\Pi_\cD}(s,a) = 0$, we have 
\[
\Pr[f^{\pi}(s,a) >0] = 
\Pr[f^{\pi}(s,a) \ge 1]\ge \frac{\probsim}{2|\cS|\abs{\actions}}.
\]
Otherwise, we have
\[
\Pr\left[f^{\pi}(s,a) \ge 2 \cdot \mu_{\probsim/(2|\cS|\abs{\actions})}^{\Pi_\cD}(s,a)\right]\ge \frac{\probsim}{2|\cS|\abs{\actions}}.
\]
\end{proof}
\begin{lemma}\label{lem:sample_complexity}
Suppose $F \ge 24 / \probsim \log (2 |\Pi| / \proboverall)$.
With probability at least $1-\proboverall/2$, Algorithm~\ref{alg:tsimall} at most interacts
\[
O\left(\frac{|\cS||\cA|}{\probsim}\cdot\log(|\cS||\cA|/\probsim)\cdot |\cS|^2|\cA|^2\log^2 H \cdot F\right)
\]
episodes  with the environment. 
\end{lemma}
\begin{proof}
Notice that our algorithm interacts with the environment only when Subroutine~\textsc{Rollout} in Algorithm~\ref{alg:tsim} is invoked. 
By Lemma~\ref{lem:double} and union bound, with probability at least $1 - \proboverall / 2$, whenever Subroutine \textsc{Rollout} is invoked, there exists $(s, a) \in \states \times \actions$ such that $\mu_{\probsim/(2|\cS|\abs{\actions})}^{\Pi_{\cD}}(s,a)$ is increased from $0$ to $1$, or $\mu_{\probsim/(2|\cS|\abs{\actions})}^{\Pi_{\cD}}(s,a)$ is increased by a factor of $2$.
Since $\mu_{\probsim/(2|\cS|\abs{\actions})}^{\Pi_{\cD}}(s,a)\le H$, with probability at least $1 - \probsim / 2$, Subroutine~\textsc{Rollout} is invoked for at most $O(|\cS||\cA|\log H)$ times.
Hence $|\Pi_\cD|= O(|\cS||\cA|\log H)$.
Finally, whenever Subroutine~\textsc{Rollout} is invoked, the algorithm samples at most $F|\Pi_\cD|\tau$ trajectories by interacting with the environment.
Therefore, with probability at least $1 - \probsim / 2$, the total number of trajectories sampled by the algorithm is upper bounded by $O(F\tau\cdot(|\cS||\cA|\log H)^2)$.
\end{proof}

\subsection{The Algorithm}
\label{sec:final}
\begin{algorithm}[t]
  \caption{\texttt{Main}\label{alg:main}}
  \begin{algorithmic}[1]
  \State \textbf{Input:} failure probability $\proboverall$, accuracy $\varepsilon$
  \State Let $\Pi_{\varepsilon / (32 H |\states|)}$ be the set of policies as defined in Definition~\ref{defn:discrete_policies}
  \State Invoke \texttt{SimAll} (Algorithm~\ref{alg:tsim}) with $\probsim = \varepsilon / 8$ and \[F = \max\{64\log(4|\Pi_{\varepsilon / (32 H |\states|)}| / \proboverall) / \varepsilon^2, 192 \log (2 |\Pi_{\varepsilon / (32 H |\states|)}| / \proboverall) / \varepsilon\}\]
  \For{each trajectory $z = (s_1, a_1, r_1), (s_2, a_2, r_2), \ldots, (s_H, a_H,r_H)$ returned by \texttt{SimAll}}
  \State Calculate
  $
  r(z) = \begin{cases}
  0 & \text{$z$ is \texttt{Fail}}\\
  \sum_{h = 1}^H r_h & \text{otherwise}
  \end{cases}
  $
  \EndFor
  \For{$\pi \in \Pi_{\varepsilon / (32 H |\states|)}$}
  \State Calculate $\hat{r}(\pi) = \frac{1}{F} \sum_{i \in [F]} r(z_i^{\pi})$ \label{line:empirical_value}
  \EndFor
  \State \Return $\argmax_{\pi \in \Pi_{\varepsilon / (32 H |\states|)}} \hat{r}(\pi)$
  \end{algorithmic}
\end{algorithm}

In this section we present our final algorithm.
The algorithm description is given in Algorithm~\ref{alg:main}.
Our algorithm invokes Algorithm~\ref{alg:tsimall} on the set of policies defined in in Definition~\ref{defn:discrete_policies} to obtain trajectories for each policy, and simply returns the policy with largest empirical cumulative reward. Now we give the formal analysis of our algorithm.

\begin{lemma}\label{lem:approx_single}
For each policy $\pi \in \Pi_{\varepsilon / (32 H |\states|)}$, for the value $\hat{r}(\pi)$ calculated in Line~\ref{line:empirical_value} of Algorithm~\ref{alg:main}, with probability at least $1 - \proboverall / (2 |\Pi_{\varepsilon / (32 H |\states|)}|)$, 
\[
\left| \hat{r}(\pi) - \expect \left[\sum_{h=1}^{H} r_h \mid \pi\right] \right| \le 5\varepsilon / 16.
\]
\end{lemma}
\begin{proof}
For those policies $\pi \in \Pi_{\cD}$, notice that $\{z_i^{\pi}\}_{i \in [F]}$ are sampled by interacting with the environment.
Since all reward values are positive and cumulative reward is upper bounded by $1$ almost surely, by Chernoff bound, 
\[
\Pr \left[ \left| \hat{r}(\pi) - \expect \left[\sum_{h=1}^{H} r_h \mid \pi\right] \right| \le \varepsilon / 8 \right] \ge 1- 2\exp(-F \varepsilon^2 / 64) \ge 1 - \proboverall / (2 |\Pi_{\varepsilon / (32 H |\states|)}|).
\]

For those policies $\pi \notin \Pi_{\cD}$, notice that $\{z_i^{\pi}\}_{i \in [F]}$ have the same distribution as $F$ independent trajectories sampled by interacting with the environment, except that at most $3\probsim /2 \cdot F = 3\varepsilon/16 \cdot F$ trajectories are replaced with \texttt{Fail}.
If all trajectories are independently sampled by interacting with the environment, by Chernoff bound, with probability at least $1 - \proboverall / (2 |\Pi_{\varepsilon / (32 H |\states|)}|)$, 
\[
\left| \hat{r}(\pi) - \expect \left[\sum_{h=1}^{H} r_h \mid \pi\right] \right| \le \varepsilon / 8.
\]
Since cumulative reward is in $[0, 1]$ almost surely, by replacing at most $ 3\varepsilon/16 \cdot F$ trajectories with \texttt{Fail}, $\hat{r}(\pi)$ is changed by at most $3\varepsilon/16$.
Therefore, with probability at least $1 - \proboverall / (2 |\Pi_{\varepsilon / (32 H |\states|)}|)$, 
\[
\left| \hat{r}(\pi) - \expect \left[\sum_{h=1}^{H} r_h \mid \pi\right] \right| \le 5\varepsilon / 16.
\]

\end{proof}

\begin{lemma}\label{lem:correct}
With probability at least $1 - \proboverall / 2$, Algorithm~\ref{alg:main} returns an $\varepsilon$-optimal policy.
\end{lemma}
\begin{proof}
By Lemma~\ref{lem:discrete_near_optimal}, there exists a $\varepsilon/4$-optimal policy $\pi' \in \Pi_{\varepsilon / (32 H |\states|)}$.
By Lemma~\ref{lem:approx_single} and a union bound over $\Pi$, with probability at least $1 - \proboverall / 2$, 
for all policy $\pi \in \Pi_{\varepsilon / (32 H |\states|)}$,
\[
\left| \hat{r}(\pi) - \expect \left[\sum_{h=1}^{H} r_h \mid \pi\right] \right| \le 5\varepsilon / 16.
\]
Let $\pi$ be the policy returned by algorithm.
Conditioned on the event mentioned above, we have
\[
\expect \left[\sum_{h=1}^{H} r_h \mid \pi\right] \ge \hat{r}(\pi) - 5\varepsilon / 16 \ge \hat{r}(\pi') - 5\varepsilon / 16 \ge \expect \left[\sum_{h=1}^{H} r_h \mid \pi'\right] - 5\varepsilon / 8 \ge \expect \left[\sum_{h=1}^{H} r_h \mid \pi^*\right] - \varepsilon .
\]
\end{proof}

Our main result, Theorem~\ref{thm:main} is a direct implication of Lemma~\ref{lem:sample_complexity} and Lemma~\ref{lem:correct}.
 
\section{Discussion and Further Open Problems}
\label{sec:dis}
This work provides an episodic, tabular reinforcement learning
algorithm whose sample complexity only depends logarithmically with
the planning horizon, thus resolving the open problem proposed in
\cite{jiang2018open}.  This result is an exponential improvement on the dependency on $H$ over existing upper bounds.
Furthermore, this works applies to 
a more general setting, where we only assume the
total  reward is bounded by one, without requiring any boundedness on
instantaneous rewards (see Section~\ref{sec:rel}).

\paragraph{Conjectured Minimax Optimal Sample Complexity for Tabular RL.}
Our upper bounds have a worse dependency on $\abs{\states},
\abs{\actions}$ and $1 / \varepsilon$ compared to existing
results. One may conjecture that this suggests there is a tradeoff
between obtaining a sublinear rate (in the model size) and obtaining a
logarithmic dependence on $H$. However,
our conjecture is that this is not the case, and that the  PAC minimax-optimal sample complexity
for episodic, tabular RL (to obtain an
$\varepsilon$-optimal policy) is of the form: 
\[
\widetilde{O}\left(\frac{\abs{\states}\abs{\actions}\poly(\log  H)}{\varepsilon^2} \right).
\]
This conjecture, if true, would show that when sample complexity is measured by
the number of episodes, then the difficulty of RL is a not a function
of the horizon, and there is a sense in which RL is no more
challenging than a contextual bandit problem.

We also conjecture a minimax optimal regret bound (rather than PAC) is
of the same form.

\paragraph{Computational Efficiency.}
The computation complexity of our simulator scales  as $\abs{\Pi}$ due
to that we need to simulate all policies in the set.
The policy set we study in this paper has cardinality exponential in $\abs{\states}$ and $\abs{\actions}$.
We leave it as an open problem is to develop a polynomial time algorithm for the setting where $r_h \ge 0 $  for $h \in [H]$ and $\sum_{h=1}^H r_h \le 1$ whose sample complexity scales only logarithmically with $H$.
One possible way is to exclude sub-optimal policies in $\abs{\Pi}$ in our simulator without explicitly evaluating them.

It is also an important open question whether or not
an upper confidence bound approach achieves the same polylogarithmic
dependence in $H$. Such an algorithm would be promising for resolving
the aforementioned conjectured minimax optimal rate.

\paragraph{Generalization and Large State Space.}
While our paper shows, from sample complexity point of view, that long planning horizon is not an issue, in practice, the state space $\states$ can be huge.
Recently, many works have established that with additional assumptions, e.g. low-rankness of the transition, functions approximations for $Q$-functions, etc, the sample complexity does not depend on $\abs{\states}$~\citep{li2011knows,wen2017efficient,krishnamurthy2016pac,jiang2017contextual,dann2018oracle,du2019provably,feng2020provably,du2019q,zhong2019pac,yang2019reinforcement,jin2019provably,du2019good,roy2019Comments,lattimore2019learning,du2020agnostic,zanette2020learning}.\footnote{Here we only focus on works where agent needs to explore the environment.
There is another line of works that require a sufficiently good exploration policy, e.g., \cite{fan2019theoretical}.
	}
However, to our knowledge, the sample complexity of all these work scales polynomially with $H$ with the only exceptions to require the transition being deterministic~\citep{wen2017efficient,du2020agnostic}.
We believe a fruitful direction is to develop algorithms for these settings with sample complexity that scales only logarithmically with $H$.

\section*{Acknowledgments}
\label{sec:ack}
We thank Alekh Agarwal, Lijie Chen, Nan Jiang, Chi Jin, Yash Nair, Zihan Zhang for useful discussions.
Ruosong Wang is supported is part by NSF IIS1763562, AFRL CogDeCON
FA875018C0014, and DARPA SAGAMORE.
Simon S. Du is supported by NSF grant DMS-1638352 and the Infosys
Membership. 
Sham Kakade gratefully acknowledges funding from the Washington Research
Foundation for Innovation in Data-intensive Discover, the ONR award
N00014-18-1-2247, and NSF Award's CCF-1703574 and CCF-1740551.

\bibliographystyle{plainnat}
\bibliography{ref}

\newpage
\appendix

\end{document}